\def\year{2021}\relax
\documentclass[letterpaper]{article} 
\usepackage{aaai21}  
\usepackage{times}  
\usepackage{helvet} 
\usepackage{courier}  
\usepackage[hyphens]{url}  
\usepackage{graphicx} 
\usepackage{natbib}  
\usepackage{caption} 
\usepackage{bm}
\usepackage{multirow}
\usepackage{subcaption}
\usepackage{amsthm}
\usepackage{xcolor}
\usepackage{amsmath}
\usepackage{booktabs}
\usepackage{graphicx}

\newtheorem{theorem}{Theorem}
\newtheorem{corollary}{Corollary}
\newtheorem{definition}{Definition}

\newtheorem*{remark}{Remark}

\newcommand{\prob}{\mathbf{P}}
\newcommand{\R}{\mathbf{R}}
\newcommand{\Rp}{\mathbf{R}_{+}}
\newcommand{\Rpp}{\mathbf{R}_{++}}
\newcommand{\N}{\mathbf{N}}

\newcommand{\tf}{f_{+}}
\newcommand{\cint}{\lambda^{\star}}

\newcommand{\tv}{\textrm{TV}}

\newcommand{\bheader}[1]{\vspace{0.5mm}\noindent\textbf{#1}}

\usepackage{pifont}

\newcommand{\specialcell}[2][c]{%
  \begin{tabular}[#1]{@{}c@{}}#2\end{tabular}}
  
\title{UNIPoint: Universally Approximating Point Processes Intensities} 
\author {
        Alexander Soen,
        Alexander Mathews,
        Daniel Grixti-Cheng,
        Lexing Xie \\
}
\affiliations {
    The Australian National University \\
    alexander.soen@anu.edu.au, alex.mathews@anu.edu.au, a500846@anu.edu.au, lexing.xie@anu.edu.au
}
\begin{document}

\maketitle

\begin{abstract}
Point processes are a useful mathematical tool for describing events over time, and so there are many recent approaches for representing and learning them. One notable open question is how to precisely describe the flexibility of point process models and whether there exists a general model that can represent {\em all} point processes. Our work bridges this gap. Focusing on the widely used event intensity function representation of point processes, we provide a proof that a class of learnable functions can universally approximate any valid intensity function. The proof connects the well known Stone-Weierstrass Theorem for function approximation, the uniform density of non-negative continuous functions using a transfer functions, the formulation of the parameters of a piece-wise continuous functions as a dynamic system, and a recurrent neural network implementation for capturing the dynamics. Using these insights, we design and implement UNIPoint, a novel neural point process model, using recurrent neural networks to parameterise sums of basis function upon each event. Evaluations on synthetic and real world datasets show that this simpler representation performs better than Hawkes process variants and more complex neural network-based approaches. We expect this result will provide a practical basis for selecting and tuning models, as well as furthering theoretical work on representational complexity and learnability.
\end{abstract}

\section{Introduction}

\begin{figure}
    \centering
    \includegraphics[width=0.45\textwidth]{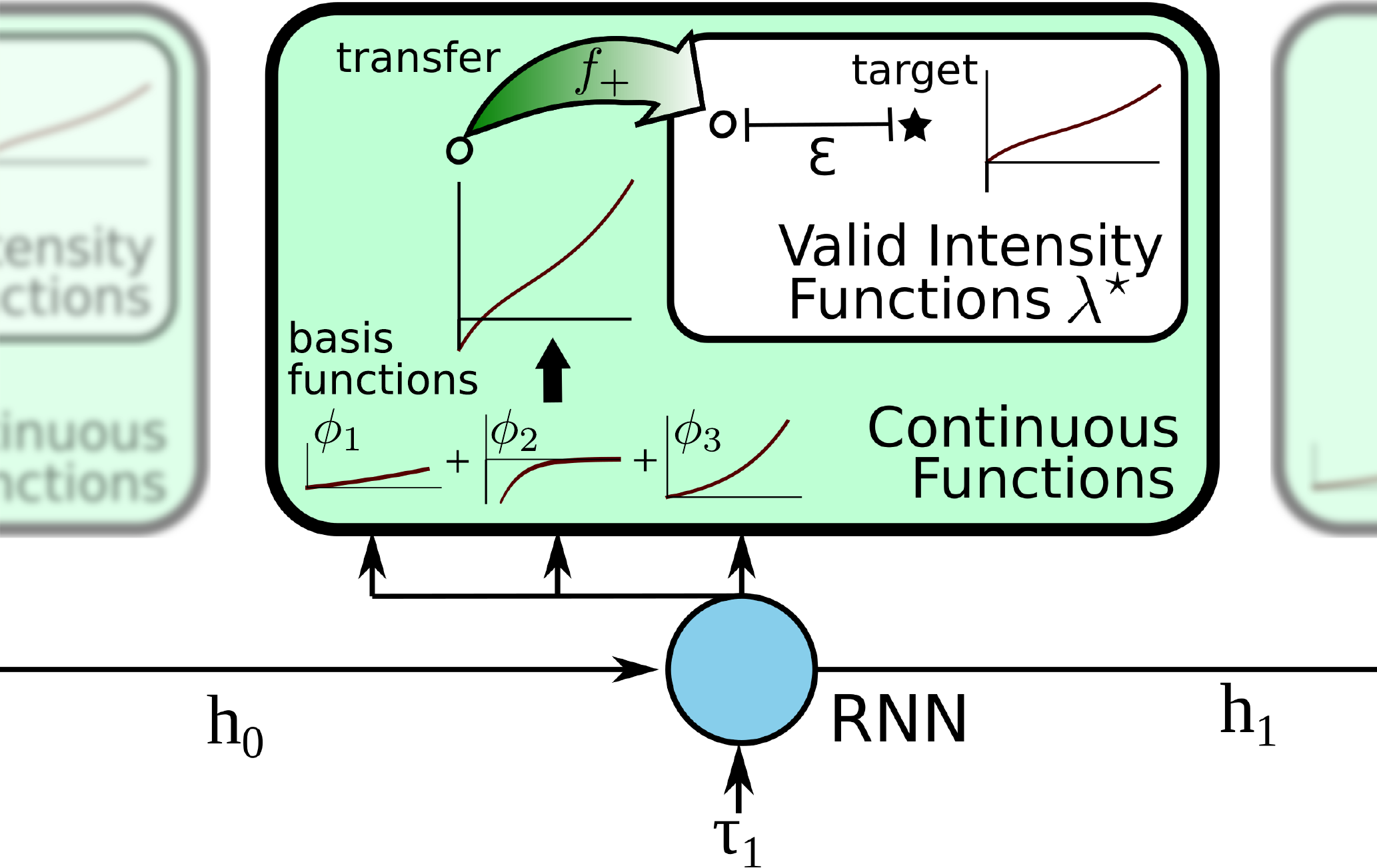}
    \caption{Overview of our method of universally approximating point processes. A RNN is used to parameterise a set of basis functions for each interarrival time \( \bm{\tau_{i}} \). Then, the sum of basis functions is used to approximate a continuous function, which is composed with a transfer function \( \bm{f_{+}} \) to universally approximate all valid intensity functions.}
    \label{fig:teaser}
\end{figure}

Temporal point processes \cite{daley2007introduction} are a preferred tool for describing events happening in irregular intervals, such as, earthquake modelling~\cite{ogata1988statistical}, social media~\cite{zhao2015seismic}, and finance~\cite{embrechts2011multivariate}. One common variant is the self-exciting Hawkes process with parametric kernel~\cite{laub2015hawkes}, which describes prior events triggering future events. However, misspecification of the kernel will likely result in poor performance~\cite{mishra2016feature}. One may ask what are the most flexible classes of point process intensity functions? How can they be implemented computationally? Does a flexible representation lead to good performance?

There is a body of literature surrounding these three questions. Multi-layer neural networks are well known for being flexible function approximators. They are able to approximate any Borel-measurable function on a compact domain~\cite{cybenko1989approximation,hornik1989multilayer}. A number of neural architectures have been proposed for point processes. The Recurrent Marked Temporal Point Process model (RMTPP)~\cite{du2016recurrent} uses Recurrent Neural Networks (RNN) to encode event history, and defines the conditional intensity function by a parametric form. Common choices of such parametric forms include an exponential function~\cite{du2016recurrent,upadhyay2018deep} or a constant function~\cite{li2018learning,huang2019recurrent}. Variants of the RNN have been explored, including NeuralHawkes~\cite{mei2017neural} that makes the RNN state a functions over time; as well as Transformer Hawkes~\cite{zuo2020transformer} and Self-attention Hawkes~\cite{zhang2019self} which uses attention mechanisms instead of recurrent units. However, a conceptual gap on the flexibility of the neural point process representation still remains. Piece-wise exponential functions~\cite{du2016recurrent,upadhyay2018deep} only encode intensities that are monotonic between events. The functional RNN representation~\cite{mei2017neural} is flexible but uses many more parameters. Transformers~\cite{zuo2020transformer,zhang2019self} are generic sequence-to-sequence function approximators~\cite{yun2020transformers}, but the functional form of the Transformer Hawkes point process intensity function is not an universal approximator.
Furthermore, intensity functions are non-negative and discontinuous at event times which means neural network approximation results cannot be applied directly. 

Recent results shed light on alternative point process representations. \citet{omi2019fully} uses a positive weight monotone neural network to learn the compensator (the integral of the intensity function). 
Although it is a generic approximator for compensators, it might assign non-zero probability to invalid inter-arrival times as the compensator can be non-zero at time zero. \cite{shchur2020intensityfree} represents inter-arrival times using normalising flow and mixture models, which can universally approximate any density. However, by defining the point process with the event density, the model cannot account for event sequences which stop naturally (see Section~\ref{sec:uaeventseq}). These approaches are promising alternatives but are not a full replacement for intensity functions, which are preferred since they are intuitive and can be superimposed. 

In this work, we propose {\em a class of neural networks that can approximate any point process intensity function to arbitrary accuracy}, along with a proof showing the role of three key constituents: a set of uniformly dense basis functions, a positive transfer function, and an approximator for arbitrary dynamic systems. We implement this proposal using RNNs, the output of which is used to parameterise a set of basis functions upon arrival of each event, as shown in Figure~\ref{fig:teaser}. Named {\em UNIPoint}, the proposed model performs well across synthetic and real world datasets in comparison to the Hawkes process and other neural variants. This work provides a general yet parsimonious representation for temporal point processes, and so forms a solid basis for future development in point process representations that incorporate rich contextual information into event models.

Our primary contributions are:

\begin{itemize}
    \item A novel architecture that can approximate any point process intensity function to arbitrary accuracy. 
    \item A theoretical guarantee for the flexible point process representation that builds upon the theory of universally approximating continuous functions and dynamic systems.
    \item UNIPoint --- the neural network implementation of the proposed architecture with strong empirical results on both synthetic and real world datasets. Reference code is available online\footnote{\url{https://github.com/alexandersoen/unipoint}}.
\end{itemize}

\subsection{Notation}

\( C(X, Y) \) denotes the class of continuous functions mapping from domain \( X \) to range \( Y \). Denote \( \R \) as the set of real numbers, \( \Rp \) as the non-negative reals and \( \Rpp \) as the strictly positive reals. Define the composition of a function \( f \) and a class of functions \( \mathcal{F} \) as \( f \circ \mathcal{F} = \{ f \circ g : g \in \mathcal{F} \} \). The sigmoid function \([1 + \exp(-x)]^{-1} \) is denoted as \( \sigma(x) \).

\section{Preliminary: Temporal Point Processes}
\label{sec:tpp_background}

A temporal point process is an ordered set of event times \( \{ t_{i} \}_{i=0}^{N} \). We typically describe a point process by its conditional intensity function \( \lambda(t \mid \mathcal{H}_{t^{-}}) \) which can be interpreted as the instantaneous probability of an event occurring at time \( t \) given event history \( \mathcal{H}_{t^{-}} \), consisting of the set of all events before time \( t \). This can be written as~\cite{daley2007introduction}:
\begin{equation}
    \label{eq:intensityeventrate}
    \lambda(t \mid \mathcal{H}_{t^{-}}) \doteq \lim_{h \downarrow 0^{+}} \frac{\prob(N[t, t+h) > 0 \mid \mathcal{H}_{t^{-}})}{h},
\end{equation}
where \( N[t_1, t_2) \) is the number of events occurring between two arbitrary times \(t_1 < t_2\). Equation~\ref{eq:intensityeventrate} restricts the conditional intensity function to non-negative functions. Given history ${\cal H}_{t^-}$, the conditional intensity is a deterministic function of time \( t \). Following standard convention, we refer to the conditional intensity function as simply the intensity function, abbreviating \( \lambda(t \mid \mathcal{H}_{t^{-}}) \) to \( \cint(t) \).

Point processes can be specified by choosing a functional form for the intensity function. For example, the Hawkes process, one of the simplest interacting point process~\cite{bacry2015hawkes}, can be defined as follows:
\begin{equation}
    \label{eq:hawkesintensity}
    \cint(t) = \mu + \sum_{t_{i} < t} \varphi(t - t_{i}),
\end{equation}
where \( \mu \) specifies the background intensity and \( \varphi(t - t_{i}) \) is the triggering kernel which characterises the self-exciting effects of prior events \( t_{i} \).

The likelihood of a point process is~\cite{daley2007introduction}
\begin{equation}
    \label{eq:likelihood}
    L = \left[ \prod_{i=1}^{N} \lambda^{\star}(t_{i}) \right] \exp \left(- \int_{0}^{T} \lambda^{\star}(s) \, ds \right),
\end{equation}
where the negated term in the exponential is known as the compensator function \( \Lambda^{\star}(t) = \int_{0}^{T} \lambda^{\star}(s) \, ds \).

\section{Universal Approximation of Intensities}
\label{sec:ppua}
To represent the influence of past events on future events, point process intensity functions \( \cint(t) \) are often continuous between events \( (t_{i-1}, \, t_{i} ] \); with discontinuities only possible at events. For example, the intensity function of the Hawkes process has discontinuities at each event,  Eq.~(\ref{eq:hawkesintensity}). Intuitively, this piece-wise continuous characterisation of the intensity function encodes the belief that the process only significantly changes its behaviour when new information (an event) is observed. As such, there are two behaviours of a point process we need to approximate: (1) the continuous intensity function segment between consecutive events, given a fixed event history; and (2) the change in the point process intensity function when an event occurs, so that we can approximate the jump dynamics between events.

We consider an intensity function \( \cint(t) \) with fixed observation period \( (0, T] \). The intensity function can be segmented by the event times of an event sequence \( (t_{0}, t_{1}], (t_{1}, t_{2}], \ldots, (t_{N-1}, t_{N}], (t_{N}, t_{N+1}] \), where \( t_{N+1} = T \). Given a piece-wise continuous intensity function, the segmented intensity function is continuous: \( u_{i}(\tau) = \cint(t) \) for \( t \in (t_{i-1}, t_{i}] \), where \( \tau = t - t_{i-1} \in (0, t_i-t_{i-1}] \). Thus to approximate the intensity function between consecutive events, we learn a function \( \hat{u}(\tau; p_{i}) \), parameterised by \( p_{i} \), to approximate {\em any} of the segmented intensity functions \( u_{i}(\tau) \), where each segment only differs in parameterisation. Then to approximate the jump dynamics of the intensity function we utilise the RNN approximation of a dynamic system, which dictates how the parameters \( p_{i} \) change over time.

To quantify the quality of an approximation, we use the uniform metric between two functions \( f, g : X \rightarrow \R \),
\begin{equation}
    \label{eq:uniformmetric}
    d(f, g) = \sup_{x \in X} \vert f(x) - g(x) \vert.
\end{equation}
This metric is the maximum difference of the two functions over a shared (compact) domain \( X \). The uniform metric has been used to prove universal approximation properties for neural networks~\cite{hornik1989multilayer,debao1993degree} and RNNs~\cite{schafer2007recurrent}. Given classes of functions \( \mathcal{F} \) and \( \mathcal{G} \), \( \mathcal{F} \) is a universal approximator of \( \mathcal{G} \) if for any \( \varepsilon > 0 \) and \( g \in \mathcal{G} \), there exists an \( f \in \mathcal{F} \) such that \( d(f, g) < \varepsilon \). An equivalently expression is: \( \mathcal{F} \) is uniformly dense in \( \mathcal{G} \).

\begin{remark}
\normalfont
Although we refer to Hawkes point processes as the primary example of a point process to approximate throughout the paper, following the example of~\citep{mei2017neural}, we note that as long as the point process has continuous intensity function between events, our approximation analysis will hold. Thus, in addition to Hawkes processes, the methods proposed in our work can approximate
point processes including self correcting processes and nonhomogenous Poisson processes with continuous densities.
\end{remark}

\subsection{Approximation Between Two Events}
\label{sec:uatwoevents}

To approximate the time shifted non-negative functions $u_i(\tau)$, we first introduce transfer functions \( \tf \) (Definition~\ref{def:transferfunction}). We then prove that the class of composed function \( \tf \circ \mathcal{F} \) preserves uniform density (Theorem~\ref{thm:transfer}). Given this theorem, we provide a method for constructing uniformly dense classes with sums of basis functions \( \Sigma(\phi) \) (Definition~\ref{def:basissum}) which are in turn uniformly dense after composing with $\tf$ (Corollary~\ref{cor:basisstoneua}). We further provide a set of suitable basis functions (Table~\ref{tab:validbasis}).

Formally, we define the \textit{M-transfer functions} which maps negative outputs of a function to positive values.

\begin{definition}
    \label{def:transferfunction}
    A function \( \tf : \R \to \Rp \) is a \textit{M-transfer function} if it satisfies the following:
    \begin{enumerate}
        \item \( \tf \) is M-Lipschitz continuous;
        \item \( \Rpp \subseteq \tf[\R] \);
        \item And \( \tf \) is strictly increasing on \( \tf^{-1}[\Rpp] \).
    \end{enumerate}
\end{definition}

Definition~\ref{def:transferfunction} provides a wide range of functions. In practice, it is convenient to use softplus function \( f_{\textrm{SP}}(x) = \log(1 + \exp(x)) \) which is a \( 1 \)-transfer function --- commonly used in other neural point processes~\cite{mei2017neural,omi2019fully,zuo2020transformer}. Alternatively, \( \tf(x) = \max(0, x) \) could be used; however, this is not differentiable at \( x = 0 \) which can cause issues in practice. Intuitively, M-transfer function are increasing functions which map to all positive values and have bounded steepness.

When a Gaussian process is used to define an inhomogenous Poisson process, the link functions serve a similar role to ensure valid intensity functions~\cite{lloyd2015variational}. However, many of these link function violate the conditions of being a M-transfer function~\cite{donner2018efficient}, i.e., the exponential link function \( \tf(x) = \exp(x) \) and squared link function \( \tf(x) = x^{2} \) are not M-Lipschitz continuous as they have unbounded derivatives; whereas the sigmoid link function \( \tf(x) = \sigma(x) \) is a bounded function (violating condition 2).

Using \( M \)-transfer functions, we can show that a uniformly dense class of unbounded functions will be uniformly dense for strictly positive functions under composition. These functions are defined with domain \( K \subset \R \), a compact subset, which can be set as \( K = [0, T] \) for intensity functions.

\begin{theorem}
    \label{thm:transfer}
    Given a class of functions \( \mathcal{F} \) which is uniformly dense in \( 
    C(K, \R) \) and a \( M \)-transfer function \( \tf \), the composed class 
    of functions \( \tf \circ \mathcal{F} \) is uniformly dense in \( C(K, 
    \Rpp) \) for any compact subset \( K \subset \R \).
\end{theorem}
\begin{proof}
    Let \( f \in C(K, \Rpp) \) and \( \varepsilon > 0 \) be arbitrary.
    Since \( \tf \) is strictly increasing and continuous on the preimage of \( \Rpp \) then \( 
    \tf^{-1} \) exists, is continuous, and restricted to subdomain \( \Rpp \). 
    Thus, there exists some \( g \in C(K, \R) \) such that \( f = \tf \circ g 
    \).
    
    As \( \mathcal{F} \) is dense with respect to the uniform metric, for \( {\varepsilon}/{M} \) there exists some \( h \in \mathcal{F} \) such that \( d(h, g) < {\varepsilon}/{M} \). Thus for any \( x \in K \),
    \begin{align*}
        \vert (\tf \circ h)(x) - f(x) \vert
        &= \vert (\tf \circ h)(x) - (\tf \circ g) (x) \vert \\
        &\leq M \vert h(x) - g(x) \vert < \varepsilon.
    \end{align*}
    We have \( d(\tf \circ h, f) < \varepsilon \).
\end{proof}

To approximate \( u_{i}(\tau) \) using Theorem~\ref{thm:transfer} we need a family of functions which are able to approximate functions in \( C(K, \R) \). We consider the family of functions consisting of the sum of basis functions \( \phi(\cdot; \, p_{j}) \), where \( p_{j} \in \mathcal{P} \) denotes the parameterisation of the basis function \( \phi \).
\begin{definition}
    \label{def:basissum}
    Denote \( \Sigma(\phi) \)  as the class of functions corresponding to the sum of basis functions \( \phi : \R \times \mathcal{P} \to \R \), with parameter space \( \mathcal{P} \), as follows:
    \begin{equation*}
        \left \{ \hat{u} : \R \to \R \mid \hat{u}(x) = \sum_{j=1}^{J} \phi(x; \, p_{j}), \; p_{j} \in \mathcal{P}, \, J \in \N \right \}.
    \end{equation*}
\end{definition}

The parameter space \( \mathcal{P} \) of a basis function is determined by the parametric form of a chosen basis function \( \phi(x; \, p_{j}) \). For example, the class composed of exponential basis functions could be defined with parameter space \( \mathcal{P} = \R^{2} \) with functions \( \{ \phi : \R \rightarrow \R \mid \phi(x) = \alpha \exp(\beta x), \; \alpha, \beta \in \R \} \). Definition~\ref{def:basissum} encompasses a wide range of function classes, including neural networks with sigmoid~\cite{cybenko1989approximation,hornik1989multilayer,debao1993degree} or rectified linear unit activations~\cite{sonoda2017neural}.

The Stone-Weierstrass Theorem provides sufficient conditions for finding basis function for universal approximation.

\begin{theorem}[Stone-Weierstrass Theorem~\cite{rudin1964principles,royden1988real}]
    \label{thm:stoneweierstrass}
    Suppose a subalgebra \( \mathcal{A} \) of \( C(K, \, \R) \), where \( K \subset \R \) is a compact subset, satisfies the following conditions:
    \begin{enumerate}
        \item For all \( x, y \in K \), there exists some \( f \in \mathcal{A} \) such that \( f(x) \neq f(y) \);
        \item For all \( x_{0} \in K \), there exists \( f \in \mathcal{A} \) such that \( f(x_{0}) \neq 0 \).
    \end{enumerate}
    Then \( \mathcal{A} \) is uniformly dense in \( C(K, \, \R) \).
\end{theorem}
Thus, by using Theorem~\ref{thm:transfer} and the Stone-Weierstrass theorem, Theorem~\ref{thm:stoneweierstrass}, we arrive at Corollary~\ref{cor:basisstoneua}, which gives sufficient conditions for basis functions \( \phi \) to ensure that \( f_{+} \circ \Sigma(\phi) \) is a universal approximator for \( C(K, \Rpp) \).

\begin{corollary}
    \label{cor:basisstoneua}
    For any compact subset \( K \subset \R \) and for any \( M \)-transfer function \( \tf \), if a basis function \( \phi(\cdot \, ; \, p) \) parametrised by \( p \in \mathcal {P} \) satisfies the following conditions:
    \begin{enumerate}
        \item \( \sum(\phi) \) is closed under product;
        \item For any distinct points \( x, \, y \in K \), there exists some \( p \in \mathcal{P} \) such that \( \phi(x; \, p) \neq \phi(y; \, p) \);
        \item For all \( x_{0} \in K \), there exists some \( p \in \mathcal{P} \) such that \( \phi(x_{0}; \, p) \neq 0 \).
    \end{enumerate}
    Then \( \tf \circ \sum(\phi) \) is uniformly dense in \( C(K, \Rpp) \).
\end{corollary}

The first condition of Corollary~\ref{cor:basisstoneua} is given such that the set of basis functions \( \sum(\phi) \) is a subalgebra of \( C(X, \R) \).
The later two conditions are the required preconditions for the Stone-Weierstrass Theorem to hold.

Given the conditions of Corollary~\ref{cor:basisstoneua}, some interesting choices for valid basis functions \( \phi(x; \, p) \) are the exponential basis function \( \phi_{\textrm{EXP}}(x) = \alpha \exp (\beta x) \) and the power law basis function \( \phi_{\textrm{PL}}(x) = \alpha (1 + x)^{-\beta} \). These basis functions are similar to the exponential and power law Hawkes triggering kernels, which have seen widespread use in many domains~\cite{ogata1988statistical,bacry2015hawkes,laub2015hawkes,rizoiu2017expecting}.

We note that the class of intensity functions in Theorem~\ref{thm:transfer} and Corollary~\ref{cor:basisstoneua} are strictly positive continuous functions. However, these results generalise to non-negative continuous functions as our definition of intensity functions permits arbitrarily low intensity in \( u_{i}(\tau) \) --- where switching from arbitrarily low intensities to zero intensity results in arbitrarily low error with respect to the uniform metric on \( (0, T] \).

In Table~\ref{tab:validbasis}, we provide a selection of interesting basis functions to universally approximate \( u_{i}(\tau) \in C(K, \Rpp) \). One should note that Corollary~\ref{cor:basisstoneua} only provides sufficient conditions, where some of the basis function in Table~\ref{tab:validbasis} do not satisfy the precondition. For example, the sigmoid basis function \( \phi_{\textrm{SIG}}(x) = \alpha \sigma(\beta x + \delta), \; (\alpha, \beta, \delta) \in \R^{3} \) does not allow \( \Sigma(\phi_{\textrm{SIG}}) \) to be closed under product and thus does not satisfy the conditions of Corollary~\ref{cor:basisstoneua}. However, the sum of sigmoid basis functions is equivalent to the class of single hidden layer neural networks~\cite{hornik1989multilayer,debao1993degree}. Thus, in additional to an appropriate transfer function it does have the universal approximation property for non-negative continuous functions through Theorem~\ref{thm:transfer}.
Additionally, other basis functions used to define point process intensity functions can be used, such as radial basis functions~\cite{tabibian2017distilling} that are not generally closed under product but have universal approximation properties~\cite{park1991universal}.

\begin{table}[t]
    \centering
    \begin{tabular}{ccc}
        \toprule
        \specialcell{Basis\\Function}& \specialcell{Functional\\Form \( \phi \)} & \specialcell{Parameter\\Space \( \mathcal{P} \)} \\
        \midrule
        \midrule
        \( \phi_{\textrm{EXP}} \)\textsuperscript{\textdagger} & \( \alpha \exp(\beta x) \) & \( (\alpha, \beta) \in \R^{2} \) \\
        \( \phi_{\textrm{PL}} \)\textsuperscript{\textdagger} & \( \alpha(1 + x)^{-\beta} \) & \( (\alpha, \beta) \in \R \times \Rp \) \\
        \( \phi_{\textrm{COS}} \)\textsuperscript{\textdagger} & \( \alpha \cos(\beta x + \delta) \) & \( (\alpha, \beta, \delta) \in \R^{3}\) \\
        \( \phi_{\textrm{SIG}} \)\textsuperscript{\(\ddagger\)} & \( \alpha \sigma(\beta x + \delta) \) & \( (\alpha, \beta, \delta) \in \R^{3} \) \\
        \( \phi_{\textrm{ReLU}} \)\textsuperscript{\( \ast \)} & \( \max(0, \alpha x + \beta) \) & \( (\alpha, \beta) \in \R^{2} \) \\
        \bottomrule
    \end{tabular}
    \caption{Basis function universal approximators for intensity functions between two consecutive events. \textdagger~indicates functions that satisfy Corollary~\ref{cor:basisstoneua}; \( \ddagger \) one proven in~\protect\cite{cybenko1989approximation}; and \( \ast \) one proven in~\protect\cite{sonoda2017neural}.}
    \label{tab:validbasis}
\end{table}

\subsection{Approximation for Event Sequences}
\label{sec:uaeventseq}

The approximations to $u_i(\tau)$ use a set of parameters, e.g. $(\alpha, \beta, \delta)$ in Table~\ref{tab:validbasis}. We denote these parameters vectors as $p_i \in {\cal P}$, and the approximated function segment as $\hat u_i(\tau; p_i)$. Since each segment $\hat u_i(\tau; p_i)$ is uniquely determined by $p_i$, and the union of all segments approximates \( \cint(t) \), we would only need to capture the dynamics in $p_i$. 

We express $p_i$ as the output of a dynamic system.
\begin{align} \label{eq:dynamic_system}
    s_{i+1} &= g(s_{i}, t_{i}) \nonumber \\
    p_{i} &= \nu(s_{i}),
\end{align}
where \( s_{i+1} \) is the internal state of the dynamic system, \( g \) updates the internal state at each step, and \( \nu \) maps from the internal state to the output.

\begin{theorem}[RNN Universal Approximation~\cite{schafer2007recurrent}] \label{thm:ua_rnn}
Let \( g : \R^{J} \times \R^{I} \rightarrow \R^{J} \) be measurable and \( \nu : \R^{J} \rightarrow \R^{n} \) be continuous, the external inputs \( x_{i} \in \R^{I} \), the inner states \( s_{i} \in \R^{J} \), and the outputs \( p_{i} \in \R \) (for \( i = 1, \ldots, N \)). Then, any open dynamic system of the form of Eq.~(\ref{eq:dynamic_system}) can be approximated by an RNN, with sigmoid activation function, to arbitrary accuracy.
\end{theorem}

Given that RNNs approximate $p_i$, we use continuity condition on basis $\phi$ and in turn $\hat u$ to show how to universally approximate an intensity function with an RNN.

\begin{theorem} \label{thm:eventseqs}
    Let \( \{t_{i} \}_{i=0}^{N} \) be a sequence of events with \( t_{i} \in [0, T] \) and \( \cint(t) \) be 
    {an intensity function.}
    {Given} a parametric family of functions \( \mathcal{F} = \{ \hat{u}(\cdot \, ; \, p) : p \in \mathcal{P} \} \) which is uniformly dense in \( C([0, T], \Rpp) \) and \( \hat{u}(x; \, p) \) continuous with respect to \( p \) for all \( x \in [0, T] \).
    Then {there} exists a recurrent neural network
    \begin{align} \label{eq:thm2}
        h_{i} &= \sigma(W h_{i-1} + v t_{i-1} + b) \nonumber \\
        \hat{p}_{i} &= A h_{i} &\textrm{ for } t \in (t_{i-1}, t_{i}] \nonumber \\
        \hat{\lambda}(t) &=\hat{u}(\tau; \hat{p}_{i}) &\textrm{ and } \tau = t - t_{i-1},
    \end{align}
    where \( \sigma \) is a sigmoid {activation function} and \( [W, v, b, A] \) are weights of appropriate shapes, such that \( \hat{\lambda}(t) \) approximates \( \cint(t) \) with arbitrary precision for all \( (0, T] \).
\end{theorem}
\begin{proof}
    Let \( \varepsilon > 0 \) be arbitrary.
    For any interval \( (t_{i-1}, t_{i} ] \), we know from the uniform density of \( \mathcal{F} \) that 
    there exists a \( p_{i} \) such that
    \begin{equation} \label{eq:pf1}
        \sup_{\tau \in [0, T]} \vert \hat{u}_{i}(\tau; p_{i}) - u_{i}(\tau) \vert \leq \frac{\varepsilon}{2}.
    \end{equation} 

    By the continuity conditions of \( \hat{u} \), it follows that for each \( p_{i} \) and any \( \tau \in [0, T] \) there exists \( \delta_{i} \) such that
    \begin{equation}
        \Vert p_{i} - \hat{p}_{i} \Vert < \delta_{i} \implies \vert \hat{u}(\tau; \, p_{i}) - \hat{u}(\tau; \, \hat{p}_{i}) \vert < \frac{\varepsilon}{2}
        \label{eq:pdelta}
    \end{equation}
    by taking the minimum over \( \delta_{\tau}\)'s in the (\(\varepsilon/2\), \(\delta_{\tau}\))-condition of continuity for all \( \tau \in [0, T] \) (where the subscript emphasises the range of \( \tau \) for fixed \( i \)).
    
    The LHS of Eq.~(\ref{eq:pdelta}) is the precision needed in our RNN approximtor for each interval \( (t_{i-1}, t_{i}] \). We take the minimum approximation discrepancy over the sequence of $\hat p_i$'s, \( \delta := \min_{i} \delta_{i} \) and use an RNN with precision \( \delta \) to bound the approximation quality due to $\hat p_{i} $'s using Theorem~\ref{thm:ua_rnn},
    \begin{equation} \label{eq:pf2}
        \sup_{\tau \in [0, T]} \vert \hat{u}(\tau; \, p_{i}) - \hat{u}(\tau; \, \hat{p}_{i}) \vert < \frac{\varepsilon}{2}.
    \end{equation}

    {Using the} triangle inequality of the uniform metric, {we can combine and bound the discrepancies due to $\hat u$ in Eq.~(\ref{eq:pf1}) and those due to $\hat p_{i}$ in Eq.~(\ref{eq:pf2}),}
    \begin{equation}
        \sup_{\tau \in [0, T]} \vert u_{i}(\tau) - \hat{u}(\tau; \, \hat{p}_{i}) \vert < \varepsilon.
        \label{eq:triangled}
    \end{equation}
    Eq.~(\ref{eq:triangled}) holds for all \( i \in \{ 1, \ldots, N \} \). Thus uniform density condition for \( \cint(t) \) also holds for the piece-wise approximator \( \hat{\lambda}(t) \) given by Eq.~(\ref{eq:thm2}) over the entire sequence.
\end{proof}

From Theorem~\ref{thm:eventseqs} and Corollary~\ref{cor:basisstoneua}, universal approximation with respect to the uniform metric follows immediately when using basis functions which are continuous with respect to their parameter space, for example Table~\ref{tab:validbasis}.

\bheader{Extensions and discussions.}  While the original work on learning the compensator function~\cite{omi2019fully} does not provide theoretical backings for its proposal, we note that
Theorem~\ref{thm:eventseqs}, combined with universal approximation capabilities of monotone neural networks~\cite{sill1998monotonic}, can be used to show that the class of monotonic (increasing) neural networks provide universal approximation for compensator functions.
The guarantee described here does not explicitly account for additional dimensions or marks. To extend Theorem~\ref{thm:eventseqs} in this manner, we consider replacing basis functions \( \phi(x) \), which has domain \( \R \), to basis functions with extended domain \( \R \times K \) where \( K \) is a compact set. For example, \( K \) can be a set of discrete finite marks in the case of approximated marked temporal point processes. The universal approximation property would then generalise as long as \( \sum (\phi) \) is dense in \( C([0,T] \times K, \Rpp) \) and continuous in the parameter space of the basis functions. Likewise, if we want to approximate a spatial point process, we let \( K = \R^{2} \) and find an appropriate set of basis functions with domain \( \R \times \R^{2} \).

It is worth mentioning two distinctions from the intensity free approach~\cite{shchur2020intensityfree}. First, although density approximation allows for direct event time sampling, the log-normal mixture representation assumes that an event will always occur on \( \Rp \) --- specifically, events cannot naturally stop. Instead, the intensity function representation allow for events to stop with probability \(1 - \prob(\tau < \infty ) = \exp\left(-\Lambda^{\star}(\infty)\right) \). In other-words, \( 1 - \prob(\tau < \infty ) \) is the probability of events not occurring in finite time, which is non-zero when the intensity function decays and stays at zero. Furthermore, the intensity free approach proposed one functional form (log-normal mixture) for approximating densities, whereas we show that a variety of basis functions all fulfil the goal of universal approximation.

\section{Implementation with Neural Networks}
\label{sec:implementation}

We propose \textit{UNIPoint}, a neural network architecture implementing a fully flexible intensity function. Let \( \{ t_{i} \}_{i=0}^{N} \) be a sequence of events with corresponding interarrival times \( \tau_{i} = t_{i} - t_{i-1} \). Let \( M \) be the size of the hidden state of the RNN, and \( \phi(\cdot; \, \cdot) \) be the chosen basis function with parameter space \( \mathcal{P} \). Let \( P \) denote the dimension of the parameter space. The approximation guarantees (given in Corollary~\ref{cor:basisstoneua}) hold in the limit of an infinite number of basis functions, in practice the number of basis functions is a hyper-parameter, denoted as \( J \). This network has four key components. 

\bheader{Recurrent Neural Network.} 
We use a simple RNN cell~\cite{elman1990finding}, though other popular variants would also work, e.g., LSTM, or GRU. 
The recurrent unit produces hidden state vector $h_i$ from $h_{i-1}$ the previous hidden state and $\tau_{i-1}$ the normalised interarrival time (divided by standard deviation): 
\begin{equation} \label{eq:unipoint_rnn}
    h_{i} = f(Wh_{i-1} + v{\tau}_{i-1} + b)
\end{equation}
Here \( W \), \( v \), \( b \), and \( h_{0} \) are learnable parameters. \( f \) is any activation function compatible with RNN universal approximation, i.e., sigmoid \( \sigma \)~\cite{schafer2007recurrent}.

\bheader{Basis Function Parameters} are generated using a 
linear transformation that maps the hidden state vector of the RNN \( h_{i} \in \R^{M} \) to parameters \( p_{i} = (p_{i1},\ldots,p_{iJ}) \),
\begin{equation}
    \label{eq:to_param}
    p_{ij} = A_{j} h_{i} + B_{j}, \quad t \in (t_{i-1}, t_{i}], \; j \in \{1, \ldots, J\}.
\end{equation}
Here \( A_{j} \) and \( B_{j} \) are learnable parameters and \( p_{ij} \in \mathcal{P} \).

Eq.~(\ref{eq:unipoint_rnn}) and Eq.~(\ref{eq:to_param}) defines the RNN which approximates a point processes' underlying dynamic system. The error contribution of these two equations is upper bounded by the sum of their individual contributions~\cite[Theorem 2]{schafer2007recurrent}.

\bheader{Intensity Function.}
{Using parameters \( p_{i1}, \ldots, p_{iJ} \)}, the intensity function with respect to time since the last event \( \tau = t - t_{i-1} \) is defined as:
\begin{equation}
    \label{eq:implintensity}
    \hat{\lambda}(\tau) = f_{\textrm{SP}} \left[ \sum_{j=1}^{J} \phi(\tau; \, {p}_{ij}) \right],
    \quad \tau \in (0, t_{i}-t_{i-1}] ,
\end{equation}
where \( f_{\textrm{SP}}(x) = \log(1 + \exp(x)) \) is the softplus function.

\bheader{Loss Function.}
We use the point process negative log-likelihood, as per Eq.~(\ref{eq:likelihood}). In most cases the integral cannot be calculated analytically so instead we calculate it numerically using Monte-Carlo integration~\cite{press2007numerical}, {see Training settings and the online appendix~\cite[Section F]{appendix}.}

Our use of RNNs to encode event history is similar to other neural point process architectures. We note that \cite{du2016recurrent} only supports monotonic intensities. Our representation is more parsimonious than \cite{mei2017neural} since the hidden states need not be functions over time, yet the output can still universally approximate any intensity function. \cite{omi2019fully} produce monotonically increasing compensator functions but can have invalid inter-arrival times.

\section{Evaluation}
\label{sec:evaluation}

We compare the performance of UNIPoint models to various simple temporal point processes and neural network based models on three synthetic datasets and three real world datasets. For the simple temporal point processes we consider self-exciting intensity functions which are piece-wise monotonic (Self-Correcting process~\cite{isham1979self} and Exponential Hawkes process~\cite{hawkes1971spectra}) and non-monotonic (Decaying Sine Hawkes process). The details of dataset preprocessing, model settings and parameter sizes can be found in the appendix~\cite[Section A and B]{appendix}.

\begin{table*}[ht]
    {\fontsize{9.0pt}{10.0pt} \selectfont
    \centering
    \begin{tabular}{|cl|lll|lll|}
        \hline
        \multicolumn{2}{|c}{Dataset} & \multicolumn{3}{c|}{Synthetic} & \multicolumn{3}{c|}{Real World} \\ 
        \cline{1-2}
        \multicolumn{2}{|l|}{Models} & \multicolumn{1}{c}{SelfCorrecting} & \multicolumn{1}{c}{ExpHawkes} & \multicolumn{1}{c}{DecayingSine} & \multicolumn{1}{|c}{MOOC} & \multicolumn{1}{c}{Reddit} & \multicolumn{1}{c|}{StackOverflow} \\
        \cline{3-8}
        \multicolumn{1}{|l|}{\multirow{4}{*}{\rotatebox[origin=c]{90}{Baseline}}} & ExpHawkes & \(-0.994 \pm .001\) & \(0.044 \pm .037\) & \(-0.838 \pm .019\) & \(3.578 \pm .060\) & \(-0.100 \pm .039\) & \(-1.031 \pm .025\) \\
        \multicolumn{1}{|l|}{} & PLHawkes & \(-0.994 \pm .001\) & \(0.036 \pm .037\) & \(-0.845 \pm .019\) & \(0.532 \pm .070\) & \(-0.787 \pm .035\) & \(-0.918 \pm .024\) \\
         \multicolumn{1}{|l|}{} & RMTPP & \(-0.776 \pm .003\) & \(0.054 \pm .038\) & \(-0.864 \pm .020\) & \(2.040 \pm .098\) & \(-0.336 \pm .031\) & \(-0.864 \pm .022\) \\
        \multicolumn{1}{|l|}{} & FullyNeural & \(-0.789 \pm .003\) & \(0.059 \pm .037\) & \(-0.833 \pm .020\) & \(4.699 \pm .054^{\dagger}\) & \(0.206 \pm .046^{\dagger}\) & \(-0.810 \pm .022\) \\
        \multicolumn{1}{|l|}{} & NeuralHawkes & \(-0.777 \pm .006^{\dagger}\) & \({0.066} \pm .037^{\dagger}\) & \({-0.821} \pm .021^{\dagger}\) & \(4.641 \pm .110\) & \(0.201 \pm .048\) & \({-0.801} \pm .023^{\dagger}\) \\
        \hline
        \multicolumn{1}{|l|}{\multirow{5}{*}{\rotatebox[origin=c]{90}{UNIPoint}}} & ExpSum & \({-0.774} \pm .008^{\ddagger}\) & \(0.056 \pm .042\) & \(-0.828 \pm .020\) & \(3.114 \pm .125\) & \(0.151 \pm .045\) & \(-0.812 \pm .023\) \\
        \multicolumn{1}{|l|}{} & PLSum & \(-0.779 \pm .006\) & \({0.064} \pm .038^{\ddagger}\) & \(-0.829 \pm .020\) & \(\mathbf{4.939 \pm .085}^{\ddagger} \) & \(0.162 \pm .046\) & \(-0.814 \pm .023\) \\
        \multicolumn{1}{|l|}{} & ReLUSum & \(-0.780 \pm .007\) & \(0.059 \pm .039\) & \(-0.828 \pm .021\) & \(4.676 \pm .075\) & \(\mathbf{0.221 \pm .046}^{\ddagger} \) & \(-0.810 \pm .023\) \\
        \multicolumn{1}{|l|}{} & CosSum & \(-0.777 \pm .008\) & \(0.062 \pm .039\) & \(-0.828 \pm .020\) & \(4.471 \pm .075\) & \(0.139 \pm .044\) & \(-0.814 \pm .023\) \\
        \multicolumn{1}{|l|}{} & SigSum & \(-0.776 \pm .007\) & \(0.064 \pm .038\) & \(-0.827 \pm .020^{\ddagger}\) & \(4.346 \pm .076\) & \(0.170 \pm .045\) & \(-0.814 \pm .023\) \\
        \multicolumn{1}{|l|}{} & MixedSum & \(-0.779 \pm .007\) & \(0.062 \pm .038\) & \(-0.828 \pm .020\) & \(4.928 \pm .085\) & \(0.201 \pm .047\) & \(-0.804 \pm .023^{\ddagger}\) \\
        \hline
    \end{tabular}
    }
    \caption{Averaged log-likelihood scores with corresponding 95\% confidence intervals. A {higher score is better}; the best of the baselines are indicated by \( \dagger \) and the best of the UNIPoint models are indicated by~\( \ddagger \). Bold indicates results when the difference between \( \dagger \) and \( \ddagger \) are {\it significantly better} (t-test \( p = 0.05 \)).}
    \label{tab:llscores}
\end{table*}

\subsection{Synthetic Datasets}

We synthesise datasets from simple temporal point process models, generating \( 2,048 \) event sequences each 
containing \( 128 \) events. This results in roughly \( 262,000 \) events, which is of the same magnitude tested in~\cite{omi2019fully}.
Self-correcting process and exponential Hawkes process datasets have previously been used in other neural point process studies~\cite{du2016recurrent,omi2019fully,shchur2020intensityfree}. We consider a decaying sine Hawkes process to test whether the models capture non-monotonic self-exciting intensity functions.
The following synthetic datasets are used:

\bheader{Self}-{\bf Correcting} Process. The intensity function is \[ \cint(t) = 
\exp \left( \nu t - \sum_{t_{i} < t} \gamma \right), \] where \( \nu = 1 \) and 
\( \gamma = 1 \).

\bheader{Exp}onential {\bf Hawkes} Process. The intensity function is a Hawkes process with exponential decaying triggering kernel, given by \[ \cint(t) = \mu + \alpha \beta \sum_{t_{i} < t} \exp(-\beta(t - t_{i})), \] where \( \mu = 0.5 \), \( \alpha = 0.8 \), and \( \beta = 1 \).

\bheader{Decaying Sine} Hawkes Process. The intensity function is a Hawkes process with a sinusoidal triggering kernel product with an exponential decaying triggering kernel: \[ \cint(t) = \mu + \gamma \sum_{t_{i} < t} (1 + \sin(\alpha(t - t_{i})) \exp(-\beta(t - t_{i})), \] where \( \mu = 0.5 \), \( \alpha = 5\pi\), \( \beta = 2 \), and \( \gamma = 1 \).

\subsection{Real World Dataset}

We further evaluate the performance of our model with three real world datasets. Although these dataset originally have marks/event types, we ignore such information to test UNIPoint.
The real world datasets used are:

\bheader{MOOC\footnote{\label{fn:data}\url{https://github.com/srijankr/jodie/}
		}.}
 A dataset of student interactions in online courses~\cite{kumar2019predicting}, 
{previously used for evaluating neural point processes~\cite{shchur2020intensityfree}}. Events 
correspond to different types of interaction, e.g., watching videos.

\bheader{Reddit\footnotemark[\value{footnote}].} A dataset of
user posts on a social media platform~\cite{kumar2019predicting}, 
{previously used for evaluating 
neural point processes~\cite{shchur2020intensityfree}}. Each event sequence corresponds to a user's 
post behaviour.

\bheader{StackOverflow}~\cite{du2016recurrent}. 
A dataset of events which consists of users gaining badges on a question-answer website. Only users with at least 40 badges between 01-01-2012 and 01-01-2014 are considered.

\subsection{Baselines}

The following traditional and neural network point process models are compared to our models.
We implement all but the NeuralHawkes baseline. We also compare to TransformerHawkes~\cite{zuo2020transformer} but the results are sensitive to model settings, the observations from which are discussed in the appendix~\cite[Section C]{appendix}.

\bheader{Exp}onential {\bf Hawkes} Process
The point process likelihood is optimised to determine parameter \( \mu \), \( 
\alpha \), and \( \beta \) in intensity function \[ \cint(t) = \mu 
+ \alpha \beta \sum_{t_{i} < t} \exp(-\beta(t - t_{i})). \]

\bheader{P}ower {\bf L}aw {\bf Hawkes} Process.
The point process likelihood is optimised to determine parameters \( \mu \), \( \alpha \), and \( \beta \) in intensity function \[ \cint(t) = \mu + \alpha \sum_{t_{i} < t} (t - t_{i} + \delta)^{-(1+\beta)} .\] The \( \delta \) parameter is fixed at \( 0.5 \) to compensate for the difficulty of the power law intensity function being infinity when \( t - t_{i} + \delta = 0 \) ~\cite{bacry2015hawkes}.

\bheader{RMTPP}~\cite{du2016recurrent}. We implement the RMTPP neural network 
architecture as a baseline. The intensity 
function of RMTPP
\begin{equation}
\label{eq:rmtpp}
\cint(t) = \exp(v^{T} h_{i} + w(t - t_{i-1}) + b)
\end{equation}
is defined with respect to the RNN hidden state \( h_{i} \). We use a RNN size of \( 48 \) for testing.

\bheader{FullyNeural}~\cite{omi2019fully}. We also implement the fully neural network point process. The integral of 
the intensity function (compensator) is defined as a neural network with RNN hidden state and 
event time input. We use a RNN size of \( 48 \) and fully connected layer of size \( 48 \) to produce the compensator.

\bheader{NeuralHawkes}\footnote{\url{https://github.com/hmeiatjhu/neural-hawkes-particle-smoothing}}~\cite{mei2017neural}. We utilise the reference implementation for NeuralHawkes~\cite{mei-2019-smoothing}, which provides a neural network architecture that encodes the decaying nature of Hawkes process exponential kernels in the LSTM of the model.
We use a LSTM size of \( 48 \) and default parameters for other model settings.

\begin{figure*}
    \centering
    \includegraphics[width=\textwidth]{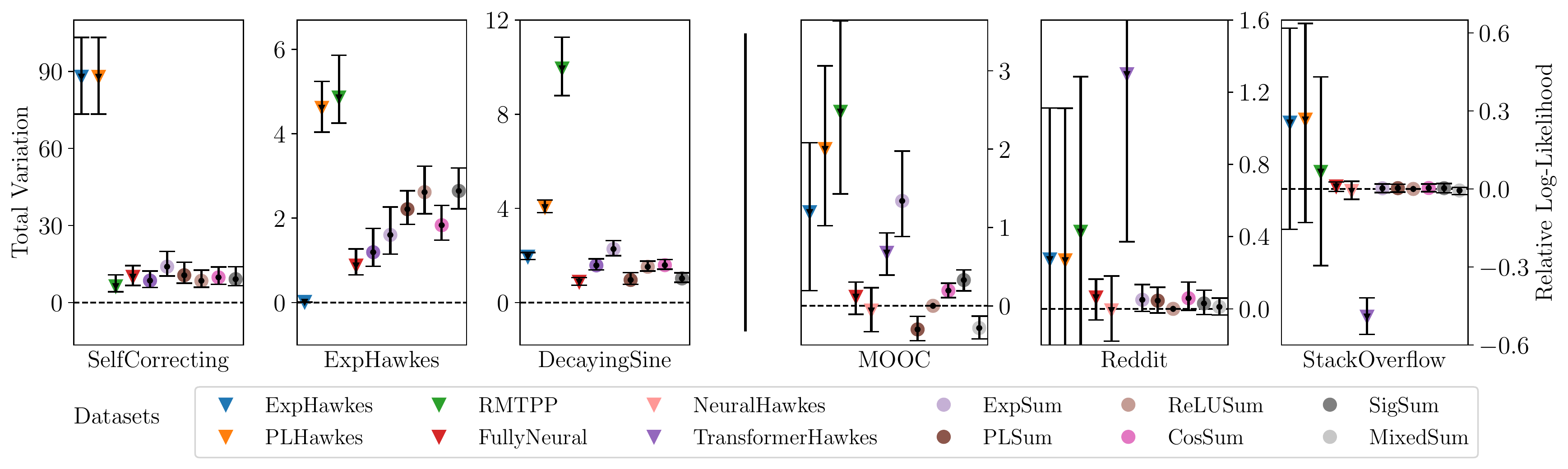}
    \caption{
    Total variation of intensity functions for synthetic 
    datasets (left) and relative log-likelihood of event 
    sequences for real world datasets 
    standardised by subtracting the 
    score of ReLUSum UNIPoint (right).
    {\em Lower score is better}. 
    Markers correspond to the mean of the score and error bars to the interquartile range.
    A missing marker indicate a mean above the visible axis range.}
    \label{fig:boxplot}
\end{figure*}

\subsection{Training Settings}

We fit a UNIPoint model for each of the five basis function types described in Table~\ref{tab:validbasis} with softplus transfer functions and \( 64 \) basis functions with learnable parameters. The mixture of basis functions, MixedSum, is used, with \( 32 \) power law and \( 32 \) ReLU basis functions. We study effects of the number of basis functions in the appendix~\cite[Section E]{appendix}. We fit models for all synthetic and real world datasets, with a \(60:20:20\) train-validation-test split. Our models are implemented in PyTorch\footnote{\url{https://pytorch.org}~\cite{paszke2017automatic}}.

During training, we use a single sample per event interval to calculate the loss function as we find using multiple samples does not improve performance, as shown in the appendix~\cite[Section F]{appendix}.
All UNIPoint models tested employ an RNN with \( 48 \) hidden units, a batch size of \( 64 \), and are
trained using Adam \cite{Kingma2014AdamAM} with \( L2 \) weight decay set to \( 10^{-5} \).
The validation set is used for 
early stopping: training halts if the validation loss does not improve by more than \( 10^{-4} \) for \( 100 \) successive mini-batches. The training for one of the real world datasets (e.g., StackOverflow) takes approximately 1 day.

We further test UNIPoint using LSTMs in Appendix G and an alternative transfer function in Appendix H.

\subsection{Evaluation Metrics}
\bheader{Holdout Log-likelihood.}
We calculate the log-likelihood of event sequences using Eq.~(\ref{eq:likelihood}). We numerically calculate the integral term with Monte-Carlo integration~\cite{press2007numerical} if it cannot be calculated analytically.

\bheader{Total Variation.}
We use total variation as it mimics the uniform metric as they both depend on the difference between the true and approximate intensity function. It is defined as
\(
    \tv(f, g) = \int \vert f(s) - g(s) \vert^{2}\, ds.
\)
Total variation can only be used on synthetic datasets where the true intensity is known. To calculate it, we use Monte-Carlo integration~\cite{press2007numerical}. We do not compute total variation for NeuralHawkes as the reference implementation does not allow the intensity function to be evaluated over fixed event histories.

\section{Results}
\label{sec:results}

Table~\ref{tab:llscores} reports log-likelihoods of all models across the three synthetic and three real world datasets. Figure~\ref{fig:boxplot} reports the total variations of intensity functions for the synthetic datasets and relative log-likelihood (calculated by subtracting the log-likelihood of UNIPoint ReLUSum) for the three real world datasets.
The total variation scores are only available for synthetic datasets since calculating the total variation requires a ground truth intensity function. 

\bheader{Synthetic datasets.} Contrasting the log-likelihood and total variation metrics reveal interesting insights about model performance. 
The SelfCorrecting dataset has a piece-wise monotonically increasing intensity function. Both metrics indicate that ExpHawkes, PLHawkes, and RMTPP under perform the other approaches by a large margin, since they are restricted to piece-wise monotone intensity functions. All UNIPoint variants perform well, achieving average likelihoods within 0.01 of each other. ExpSum is the best variant, possibly due to its exponential shape matching that of the ground-truth SelfCorrecting intensity function.

For the ExpHawkes dataset, the ExpHawkes baseline has the lowest total variation (close to zero, as expected) but not the best holdout log-likelihood.
This indicates that models with good log-likelihood scores still have the potential to overfit given the wrong intensity function representation.
Despite UNIPoint's guarantees with infinite basis functions, ExpSum shows significantly better total variation scores than other UNIPoint models here --- showing that, selection of basis functions for specific datasets is important.

For DecayingSine, the intensity between events are non-monotonic. All UNIPoint variants perform comparably on both the log-likelihood and total variation metric. The FullyNeural approach performs comparably with the UNIPoint variants on total variation, but is inferior on log-likelihood. This is likely due to it assigning non-zero probabilities to negative event times. NeuralHawkes has the best log-likelihood for this dataset, but the difference with respect to SigSum is not significant. 

In addition, we visualise intensity functions learnt by UNIPoint and other approaches, see the appendix~\cite[Section D]{appendix}. The neural baseline models learn similar intensity functions to UNIPoint in ExpHawkes. However in the case of the MOOC dataset, RMTPP learns an intensity function that is different to those learnt by the other neural point processes. Meanwhile, FullyNeural does not exhibit strong decaying components in the intensity function.

\bheader{Real-world datasets.} For all three real-world datasets, baselines ExpHawkes, PLHawkes, and RMTPP significantly under-perform in comparison to the rest of the approaches. This likely occurs due to their inability to support non-monotone intensity functions in inter-event intervals.

We observe that UNIPoint variants are significantly better than the baselines for MOOC and Reddit. UNIPoint is second best (to NeuralHawkes) on StackOverflow dataset, but the difference is not statistically significant. 
NeuralHawkes performs strongly on the StackOverflow dataset, potentially because it has the closest architecture to the UNIPoint ExpSum variant, while also being more complex. In particular NeuralHawkes has time decaying hidden states and LSTM recurrent units rather than the a perceptron recurrent unit and vector-formed hidden state of UNIPoint.
The StackOverflow dataset has a longer average sequence length than MOOC and Reddit, which would advantage the LSTM recurrent units over the standard RNN --- since the RNN is more likely to suffer from vanishing or exploding gradients than the LSTM which allows for long-term dependencies~\cite{hochreiter1997lstm}. Details on dataset characteristics can be found in the appendix~\cite[Section A]{appendix}.
One peculiar result is the performance of ExpSum in the MOOC dataset. The reason for the poor performance is that the exponential basis function is unstable with large interarrival times, which can cause numeric overflow or underflow.
The performance of UNIPoint variants depend greatly on the particular basis function used for each dataset. We find that no single type of basis function ensures that a UNIPoint model performs best over all datasets. For example, in the MOOC dataset, ignoring ExpSum, the UNIPoint models have log-likelihood scores from \( 4.346 \pm 0.076 \) to \( 4.939 \pm 0.085 \).

Using mixture of basis function, MixedSum provides good overall performance. Among the UNIPoint variants, it is either the best or a close second across all datasets, suggesting that using different types of basis functions improves model flexibility in practice even with a fixed parameter budget. 
We also observe an improvement in performance when more basis functions are used, see Appendix E.

Overall, our evaluations demonstrate the power of UNIPoint for modelling complex intensity function that are not piece-wise monotone. Results on real-world datasets show models with flexible intensity functions outperform Hawkes processes. Open questions remain on which neural architectures, among the ones with universal approximation power, strike the best balance of representational power, parsimony, and learnability.

\section{Conclusion}

We develop a new method for universally approximating the conditional intensity function of temporal point processes. This is achieved by breaking down the intensity function into piece-wise continuous functions and approximating each segment with a sum of basis functions, followed by a transfer function. We also propose UNIPoint, a neural implementation of the approximator. Evaluations on synthetic and real world benchmarks demonstrates that UNIPoint consistently outperform the less flexible alternatives. 
Future work include:
investigating methods for selecting and tuning different basis functions and further theoretical work on representation complexity, expressiveness and learnability.

\section*{ Acknowledgments}
This research was supported in part by the Australian Research Council Project DP180101985 and AOARD Project 20IOA064. This research was supported by use of the Nectar Research Cloud, a collaborative Australian research platform supported by the National Collaborative Research Infrastructure Strategy (NCRIS).

\bibliography{my_bib}

\clearpage

\appendix

\section{Dataset Preprocessing}
\label{app:dataset_preprocessing}

We use two different types of preprocessing steps.

For the first, we normalise the interarrival time inputs used for the RNN. We only use this normalisation for models which we implement. 
Specifically, over all datasets, inputs to the RNN are standardised by the training set mean and standard deviation of interarrival times. Eq.~(\ref{eq:unipoint_rnn}) with the preprocessing included is
\begin{align*}
    h_{i} &= f(Wh_{i-1} + v{\hat\tau}_{i-1} + b) \\
    {\hat\tau}_{i-1} &= \frac{\tau_{i-1} - \mu}{\sigma}
\end{align*}
where \( \mu \) is the average interarrival time over all sequences and \( \sigma \) is the standard deviation of the interarrival time over all sequences.

The starting token of the RNN \( h_{0} \in \R^{N} \) is a learnable parameter vector, where \( \tau_{0} = 0 \) to calculate the first hidden state \( h_{1} \).

For the second we normalise the interarrival times for evaluating the intensity function, and thereby the log-likelihood calculation. Instead of using plain interarrival times, we divide by the training standard deviation. We only normalise inputs on the real world datasets. Training the UNIPoint model (and some other neural network models) without this normalisation, numeric errors often cause issues. We did not experience this issue in the synthetic dataset so we did not apply normalisation.

\section{Model Details}
\label{app:model_details}

In Table~\ref{tab:num_params}, we detail the number of learnable parameters for each of the models tested.

\begin{table}[ht]
    \centering
    \begin{tabular}{ll}
        \toprule
        Model & \# Params. \\
        \midrule
        ExpHawkes & 2 \\
        PLHawkes & 3 \\
        RMTPP & 2498 \\
        FullyNeural & 7249 \\
        NeuralHawkes & 32832 \\
        ExpSum & 8768 \\
        PLSum & 8768 \\
        ReLUSum & 11904 \\
        CosSum & 8768 \\
        SigSum & 11904 \\
        MixedSum & 10336 \\
        \bottomrule
    \end{tabular}
    \caption{The number of learnable parameter for fitted models.}
    \label{tab:num_params}
\end{table}

\bheader{RMTPP. } We use a 48 dimension RNN with a single layer. The hidden state of the RNN is directly used to define the intensity function, as per Eq.~(\ref{eq:rmtpp}), yielding a small number of total learnable parameters. The formulation of the intensity function can be considered as a restricted instance of the ExpSum UNIPoint model with only 1 basis function.

\bheader{FullyNeural. } We use a 48 size dimension RNN with a single layer. The compensator function is computed with a one hidden layer neural network. The 48 dimension RNN hidden state is an input to the one hidden layer neural network, where the hidden layer has a size of 48 as well. The output dimension of the neural network is 1, such that it approximates the compensator value of a point process.

\bheader{NeuralHawkes. } We use the updated implementation provided by~\cite{mei-2019-smoothing}. We use 48 LSTM cells, a batch size of 64, a learning rate of 1e-3, and trained for a maximum of 500 epochs. The best performing model out of the \( 500 \) epochs is saved as the final model, where performance is measured in log-likelihood of the validation set.

\section{Transformer Hawkes}
\label{app:transformer_hawkes}

We evaluate the recently proposed TransformerHawkes model~\cite{zuo2020transformer} in our setting. We used the implementation released by the authors, and the following general learning setting. For each of the settings, we use the default parameters for dropout (0.1), learning rate (1e-4), and smoothness (0.1); and we use a maximum of 200 training epochs.

We use two model setting when testing the TransformerHawkes, corresponding to a small and large parameter set. The small setting corresponds to using 4 attention heads, 4 layers, model dimension of 16, encoding RNN dimension of 8, inner dimension of 16, key size of 8, and value size of 8; resulting in 11745 learnable parameters. The large setting corresponds using 4 attention heads, 4 layers, model dimension of 32, RNN dimension size 16, inner dimension of 32, key size of 16 and value size of 16; resulting in 45761 learnable parameters.

Table~\ref{tab:th} summaries the log-likelihood results. We observe that across the 3 synthetic and 3 real world datasets, it either underperforms or outperforms all other approaches by a large margin. We posit three possible reasons, while still working towards a better understanding of this result: (1) model sizes evaluated here is very different from those in the paper, which was in the magnitude of 100K parameters; (2) a different objective function used than the other models, where an event time prediction tasks contribute to the loss; and (3) it is sensitive to training and hyper-parameter setting that we have yet to identify.

The paper uses model settings which result in significantly more learnable parameters, where the largest settings tested are roughly 1000K learnable parameters large. As we are operating in smaller parameter settings, the degradation of performance .
Additional to the scale of the learnable parameters, the loss function of TransformerHawkes is not in the same form as the other models examined in the main text of the paper. In particular, TransformerHawkes uses the RMSE of event time prediction in addition to the log-likelihood of point processes in the loss functions, as per Eq.~(\ref{eq:likelihood}). To accommodate for the extra component of the loss function, a specific event time prediction layer is used. This could account for the highly variable performance of TransformerHawkes in comparison to all other models.

\section{Fitted Intensity Functions}
\label{app:fitted_intensity_functions}

We present the ExpHawkes and MOOC fitted intensity functions in Figure~\ref{fig:event_exphawkes_intensity} and Figure~\ref{fig:event_mooc_intensity} respectively. We use dotted lines in the ExpHawkes plot so we can see the underlying true intensity as we have a ground-truth point process for the generated synthetic dataset.

In Figure~\ref{fig:event_exphawkes_intensity}, we can see that all plotted fitted models fit the true intensity quite closely. However, when many events occur in succession, some of the models deviate from the true intensity function. In particular, FullyNeural has some erratic behaviour, where the intensity is under-estimated before events around time 8. RMTPP also has some errors, however they are not as visible as FullyNeural despite having a lower log-likelihood in Table~\ref{tab:llscores}. For our UNIPoint processes, PLSum slighly underestimates the intensity function when it decays, as seen around time 6. MixedSum seems to be able to fit the true intensity better than PLSum (and the other neural models) by having a mixture of ReLU and powerlaw basis functions.

In Figure~\ref{fig:event_mooc_intensity}, we can see that for the specified event sequence, the RMTPP fails to learn similar intensity function shapes to the other neural models. The PLSum and MixedSum models have similar shapes with strong decaying intensity functions after events, similar to the traditional intensity functions. The FullyNeural model however, does not have decaying components in the functional form of its intensity function, thus the change after an event is more smooth. The FullyNeural intensity function also rises less steeply after an event when compared to the UNIPoint models.

\section{Choosing the Number of Basis Functions}
\label{app:num_basis_functions}

Figure~\ref{fig:varymooc} shows the log-likelihood for PLSum on the MOOC dataset for a different numbers of basis functions. More basis functions leads to better log-likelihood scores due to being more flexible in representation. The log-likelihood scores plateaus from 16 basis functions onward. We choose 64 basis functions for our experiments. 

\begin{figure}[ht]
    \centering
    \includegraphics[width=0.45\textwidth]{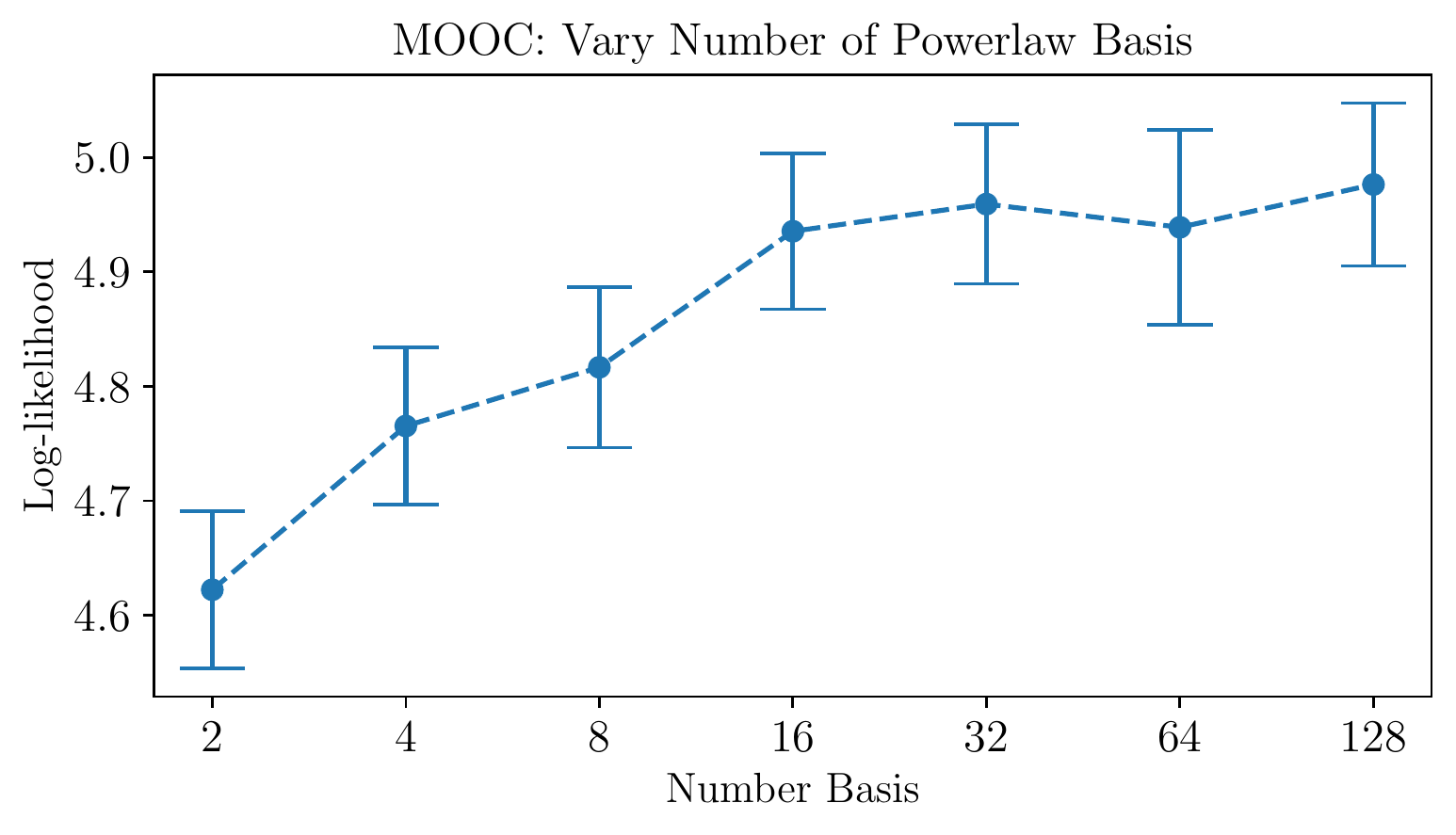}
    \caption{Varying number of basis function for PLSum in the MOOC dataset. Points indicate the average log-likelihood and the error bars indicate the 95\% confidence interval.}
    \label{fig:varymooc}
\end{figure}

\section{Choosing the Number of Monte Carlo Samples}
\label{app:num_mc_samples}

We utilise Monte Carlo integration to calculate the compensator term in the point process log-likelihood, Eq.~(\ref{eq:likelihood}). Table~\ref{tab:mc} shows the differences in ReLUSum log-likelihood for the MOOC dataset over different number of MC integration points, which are used to calculate the point process compensator, or the second term in Eq.~(\ref{eq:likelihood}). Instead of considering the average log-likelihood, as per Table~\ref{tab:llscores}, we calculate the relative log-likelihood scores between different number of Monte Carlo samples. Each row contains the average absolute difference with a 256 sample approximation of the fitted ReLUSum log-likelihood loss function. 
We can see that the effect size is in the third decimal point, smaller than the standard deviations.
The number of integration points contribute highly to computation cost in training. Thus, in training we use a single Monte Carlo sample to calculate the loss function.

\section{UNIPoint with LSTM}
\label{app:unipoint_lstm}

Additional to using a standard RNN units in our UNIPoint architecture, we also trying using an LSTM --- replacing Eq.~\eqref{eq:unipoint_rnn}. We limit our testing to the best performing UNIPoint models on real world datasets (Table~\ref{tab:llscores}). Results of holdout log-likelihood scores using LSTMs are documented in Table~\ref{tab:llscores_lstm}. Interestingly, we see a drop in log-likelihood over UNIPoint processes on the MOOC dataset. Despite this, we see quite a large increase in log-likelihood for the other two real world datasets. In the Reddit dataset, the performance of the UNIPoint model which is significantly better than the baseline models is further increased with an LSTM. For the StackOverflow dataset, although performance of UNIPoint increases we still are not significantly better than that of the best baseline, i.e., NeuralHawkes.

\section{UNIPoint with Alternative Transfer Function}
\label{app:unipoint_alttf}

Alternatively to the softplus transfer function we use throughout the paper, we also consider the following 1-transfer function
\begin{equation*}
    \tf(x) = \max(\sigma(x), x),
\end{equation*}
where \( \sigma(x) \) is the sigmoid activation function.

We report the holdout log-likelihood in Table~\ref{tab:llscores_alttf}. The change in performance from replacing the softplus function is not consistent across datasets. ExpSum and SigSum UNIPoint models has an increase in performance across all datasets, with ExpSum being significantly better in MOOC than its softplus counterpart. Meanwhile, the CosSum UNIPoint model gets worse across all datasets --- with remaining models having a mix of better and worse log-likelihood scores depending on the dataset evaluated over. This indicates that choice of transfer function is sensitive to the basis functions used and the dataset chosen.

\begin{table*}[t]
    {\fontsize{9.0pt}{10.0pt} \selectfont
    \centering
    \begin{tabular}{|l|lll|lll|}
        \hline
        \multicolumn{1}{|c}{Dataset} & \multicolumn{3}{c|}{Synthetic} & \multicolumn{3}{c|}{Real World} \\
        \cline{1-1}
        \multicolumn{1}{|l|}{TransformerHawkes} & \multicolumn{1}{c}{SelfCorrecting} & \multicolumn{1}{c}{ExpHawkes} & \multicolumn{1}{c}{DecayingSine} & \multicolumn{1}{|c}{MOOC} & \multicolumn{1}{c}{Reddit} & \multicolumn{1}{c|}{StackOverflow} \\
        \cline{2-7}
        Small & \({-0.847} \pm .002\) & \(-0.097 \pm .037\) & \(-0.505 \pm .019\) & \(3.636 \pm .075\) & \({0.627} \pm .032\) & \(-0.509 \pm .018\) \\
        Large & \(-0.926 \pm .001\) & \({0.168} \pm .036\) & \({-0.434} \pm .018\) & \({4.095} \pm .060\) & \(-1.353 \pm .097\) & \({-0.320} \pm .019\) \\
        \hline
    \end{tabular}
    }
    \caption{Averaged log-likelihood scores with corresponding 95\% confidence intervals for TransformerHawkes.}
    \label{tab:th}
\end{table*}

\begin{table*}[ht!]
    \centering
    {\fontsize{9.0pt}{10.0pt} \selectfont
    \begin{tabular}{l|ccccccccc}
        \toprule
        \# MC Points & 1 & 2 & 4 & 8 & 16 & 32 & 64 & 128 & 256 \\
        \midrule
        \multirow{2}{*}{LL \( \Delta \)}
        & \( 0.005 \)
        & \( 0.003 \)
        & \( 0.002 \)
        & \( 0.001 \)
        & \( 0.001 \)
        & \( 0.001 \)
        & \( 0.001 \)
        & \( 0.000 \)
        & \( - \) \\
        & \( \pm 0.019 \)
        & \( \pm 0.009 \)
        & \( \pm 0.011 \)
        & \( \pm 0.005 \)
        & \( \pm 0.004 \)
        & \( \pm 0.004 \)
        & \( \pm 0.002 \)
        & \( \pm 0.002 \)
        & \( - \) \\
        \bottomrule
    \end{tabular}
    }
    \caption{MOOC log-likelihood differences with corresponding standard deviation for ReLUSum. Values correspond to the absolute difference with the 256 MC sample approximation. }
    \label{tab:mc}
\end{table*}
\begin{figure*}[ht!]
    \centering
    \includegraphics[width=0.8\textwidth]{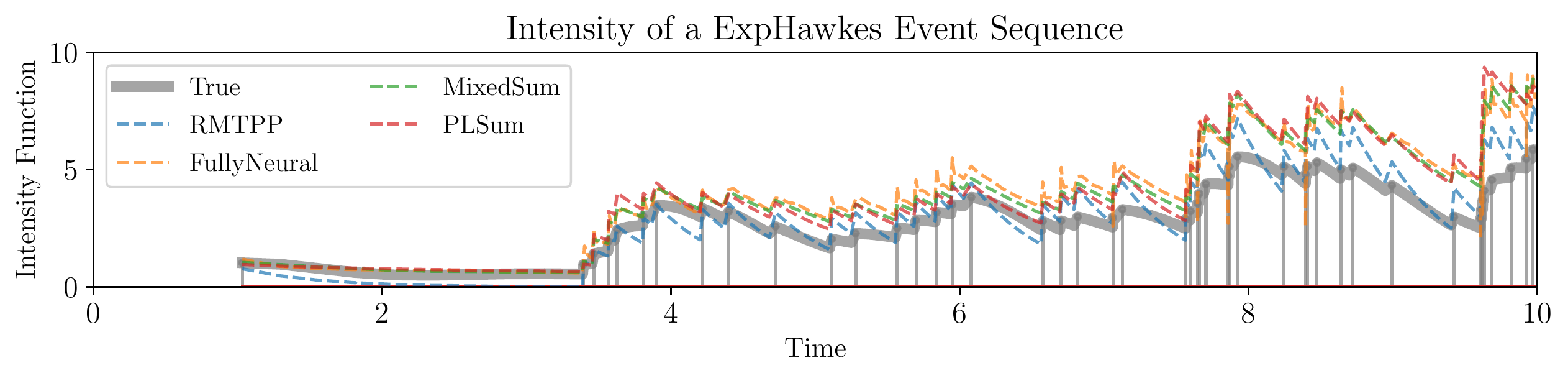}
    \caption{Intensity function of RMTPP, FullyNeural, PLSum, and MixedSum for a single ExpHawkes event sequence.}
    \label{fig:event_exphawkes_intensity}
\end{figure*}
\begin{figure*}[ht!]
    \centering
    \includegraphics[width=0.8\textwidth]{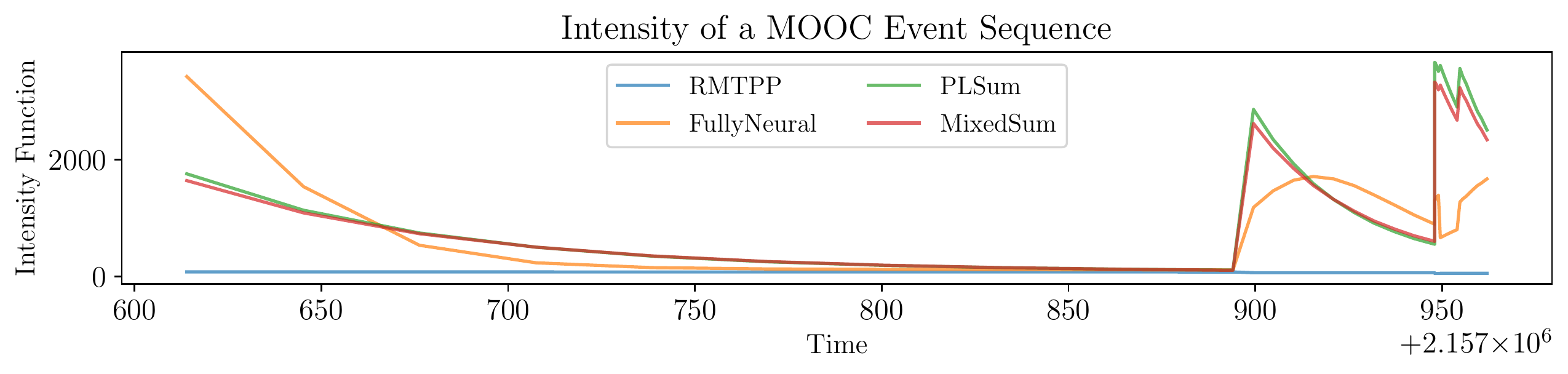}
    \caption{Intensity function of RMTPP, FullyNeural, PLSum, and MixedSum for a single MOOC event sequence.}
    \label{fig:event_mooc_intensity}
\end{figure*}

\begin{table*}[ht]
    {\fontsize{9.0pt}{10.0pt} \selectfont
    \centering
    \begin{tabular}{|cl|llllll|}
        \hline
        \multicolumn{2}{|c}{Dataset} & \multicolumn{6}{c|}{Real World} \\
        \cline{1-2}
        \multicolumn{2}{|l|}{Models} & \multicolumn{2}{|c}{MOOC} & \multicolumn{2}{c}{Reddit} & \multicolumn{2}{c|}{StackOverflow} \\
        \cline{3-8}
        \multicolumn{1}{|l|}{\multirow{3}{*}{\rotatebox[origin=c]{90}{UNIP.}}} & PLSum    & \(\mathbf{4.939 \pm .085}^{\ddagger} \) & \(\rightarrow \color{red}{4.934 \pm .086} \) & \(0.162 \pm .046\) & \(\rightarrow \color{green}{0.212 \pm .048}\)                      & \(-0.814 \pm .023\) & \( \rightarrow \color{green}{-0.791 \pm .023}\) \\
        \multicolumn{1}{|l|}{}                                                    & ReLUSum  & \(4.676 \pm .075\) & \( \rightarrow \color{red}{4.651 \pm .074}\) & \(\mathbf{0.221 \pm .046}^{\ddagger}\) & \(\rightarrow \color{green}{0.241 \pm .047}\) & \(-0.810 \pm .023\) & \( \rightarrow \color{green}{-0.797 \pm .023}\) \\
        \multicolumn{1}{|l|}{}                                                    & MixedSum & \(4.928 \pm .085\) & \( \rightarrow \color{red}{4.918 \pm .083}\)                      & \(0.201 \pm .047\) & \( \rightarrow \color{green}{0.238 \pm .048}\)                      & \(-0.804 \pm .023^{\ddagger}\) & \( \rightarrow \color{green}{-0.792 \pm .023}\) \\
        \hline
    \end{tabular}
    }
    \caption{UNIPoint with LSTM: Averaged log-likelihood scores with corresponding 95\% confidence intervals. A {higher score is better}.}
    \label{tab:llscores_lstm}
\end{table*}

\begin{table*}[ht]
    {\fontsize{9.0pt}{10.0pt} \selectfont
    \centering
    \begin{tabular}{|cl|llllll|}
        \hline
        \multicolumn{2}{|c}{Dataset} & \multicolumn{6}{c|}{Real World} \\
        \cline{1-2}
        \multicolumn{2}{|l|}{Models} & \multicolumn{2}{|c}{MOOC} & \multicolumn{2}{c}{Reddit} & \multicolumn{2}{c|}{StackOverflow} \\
        \cline{3-8}
        \multicolumn{1}{|l|}{\multirow{5}{*}{\rotatebox[origin=c]{90}{UNIPoint}}} & ExpSum   & \(3.114 \pm .125\)                      & \(\rightarrow {\color{green}3.720 \pm .103}\) & \(0.151 \pm .045\)                      & \(\rightarrow \color{green}{0.156 \pm .045}\) & \(-0.812 \pm .023\) & \( \rightarrow \color{green}{-0.813 \pm .023} \) \\
        \multicolumn{1}{|l|}{}                                                    & PLSum    & \(\mathbf{4.939 \pm .085}^{\ddagger} \) & \(\rightarrow \color{red}{4.937 \pm .085} \) & \(0.162 \pm .046\) & \(\rightarrow \color{green}{0.187 \pm .047}\)                      & \(-0.814 \pm .023\) & \( \rightarrow \color{red}{-0.802 \pm .023}\) \\
        \multicolumn{1}{|l|}{}                                                    & ReLUSum  & \(4.676 \pm .075\) & \( \rightarrow \color{red}{4.672 \pm .075}\) & \(\mathbf{0.221 \pm .046}^{\ddagger}\) & \( \rightarrow \color{green}{0.222 \pm .049}\) & \(-0.810 \pm .023\) & \( \rightarrow \color{green}{-0.806 \pm .023}\) \\
        \multicolumn{1}{|l|}{}                                                    & CosSum   & \(4.471 \pm .075\)                      & \( \rightarrow \color{red}{3.353 \pm .153} \) & \(0.139 \pm .044\)                      & \( \rightarrow \color{red}{0.151 \pm .045} \) & \(-0.814 \pm .023\) & \( \rightarrow \color{red}{-0.814 \pm .022} \) \\
        \multicolumn{1}{|l|}{}                                                    & SigSum   & \(4.346 \pm .076\)                      & \( \rightarrow \color{green}{4.373 \pm .077} \) & \(0.170 \pm .045\)                      & \( \rightarrow \color{green}{0.191 \pm .045} \) & \(-0.814 \pm .023\) & \( \rightarrow \color{green}{-0.813 \pm .023} \) \\
        \multicolumn{1}{|l|}{}                                                    & MixedSum & \(4.928 \pm .085\) & \( \rightarrow \color{green}{4.936 \pm .085}\)                      & \(0.201 \pm .047\) & \( \rightarrow \color{green}{0.220 \pm .048}\)                      & \(-0.804 \pm .023^{\ddagger}\) & \( \rightarrow \color{red}{-0.805 \pm .023}\) \\
        \hline
    \end{tabular}
    }
    \caption{UNIPoint with transfer function \( \max(\sigma(x), x) \). Averaged log-likelihood scores with corresponding 95\% confidence intervals. A {higher score is better}.}
    \label{tab:llscores_alttf}
\end{table*}

\end{document}